\documentclass[preprint,12pt]{elsarticle}
\usepackage{amssymb}
\usepackage{amssymb}
\usepackage{amsmath}
\usepackage{amsthm}
 \newtheorem{theorem}{Theorem}[section]
 
\usepackage{lscape}
\usepackage{algorithm}
\usepackage{algorithmic}
\usepackage{graphicx}
\usepackage{multirow}
\newtheorem{example}{Example}[section]
\newtheorem{definition}{Definition}
\newtheorem{remark}{Remark}
\journal{Physics A}

\begin{document}

\begin{frontmatter}

\title{A Visibility Graph Averaging Aggregation Operator}

\author[label1]{Shiyu Chen}
\author[label2]{Yong Hu}
\author[label3]{Sankaran~Mahadevan}
\author[label1,label3]{Yong Deng\corref{label4}}
\cortext[label4]{Corresponding author: Yong Deng, School of Computer and Information Science, Southwest University, Chongqing, 400715, China. Email address: ydeng@swu.edu.cn; prof.deng@hotmail.com. Tel.: +86 23 6825 4555, Fax: +86 23 6825 4555}

\address[label1]{School of Computer and Information Science, Southwest University, Chongqing, 400715, China.}
\address[label2]{Institute of Business Intelligence and Knowledge Discovery, Guangdong University of Foreign Studies, Guangzhou 510006, China}
\address[label3]{School of Engineering, Vanderbilt University, Nashville, TN, 37235, USA.}

\begin{abstract}
The problem of aggregation is considerable importance in many disciplines. In this paper, a new type of operator called visibility graph averaging (VGA) aggregation operator is proposed. This proposed operator is based on the visibility graph which can convert a time series into a graph. The weights are obtained according to the importance of the data in the visibility graph. Finally, the VGA operator is used in the analysis of the TAIEX database to illustrate that it is practical and compared with the classic aggregation operators, it shows its advantage that it not only implements the aggregation of the data purely, but also conserves the time information, and meanwhile, the determination of the weights is more reasonable.

\end{abstract}

\begin{keyword}
The visibility graph \sep Aggregation operator \sep The ordered weighted averaging (OWA) aggregation \sep Forecasting

\end{keyword}

\end{frontmatter}

\section{Introduction}
\label{Introduction}
Aggregation is a process of combining several numerical values into a single one which exists in many disciplines, such as image processing \cite{soria2001new,liu2013color}, pattern recognition \cite{soda2009aggregation,russo1999fire}, decision making \cite{merigo2013induced,merigo2011decision,merigo2011induced} and so forth \cite{marek1996iterative,liu2010owa,cheng2013owa,smutna2004graded,wei2011prediction,wang2011discrete}. To obtain a consensus quantifiable judgements, some synthesizing functions have been proposed.

For example, arithmetic mean, geometric mean, median can be regarded as a basic class, because they are often used and very classical. However, these operators is not able to model an interaction between criteria. For having a representation of interaction phenomena between criteria, Fuzzy measures have been proposed by Sugeno in 1974 \cite{sugeno1974theory}. Two main classes of the fuzzy measures are Choquet and Sugeno integrals. Choquet and Sugeno integrals are idempotent, continuous and monotone operators. The ordered weighted average (OWA) operators are a particular case of discrete Choquet integrals. The OWA operators were introduced by Yager in \cite{yager1988ordered} to provide an aggregation which lies in between the ``and" and the ``or " operators. The ``and" (t-norms) and the ``or" (t-conorms) operators are generalizations of the logical aggregation operators which are two specialized aggregation families. Above operators try to look for giving a ``middle value", but the t-norms and the t-conorms can compute the intersection and union of fuzzy sets.

However, to the best of our knowledge, these operators do not consider the influence of time specially and the time factor should not be ignored in some areas such as economics, space science, weather forecast and so forth. In this paper, a novel aggregation operator called visibility graph averaging (VGA) aggregation operator is proposed which can deal with the time series effectively.

This paper is inspired by the pioneering work the visibility graph \cite{lacasa2008time} which builts a natural bridge between complex network theory and time series. In the visibility graph, the values of a time series are plotted by using vertical bars. These vertical bars are regarded as landscapes, and every bar is linked with others that can be seen from the top of the considered one, then the associated graph is obtained. According to the study, it is found that the structure of the time series is conserved in the graph topology. For example, periodic series convert into regular graphs, random series convert into random graphs, and fractal series convert into scale-free graphs. Until now, the visibility graph has been applied in economic \cite{wang2012visibility}, geology \cite{telesca2013investigating,telesca2012analysis}, praxiology \cite{fan2012fractal}, biological system \cite{ahmadlou2010new,ahmadlou2012improved} and so forth \cite{yang2009visibility,gutin2011characterization,liu2010statistical,lacasa2009visibility,qian2010universal,donner2012visibility}. The proposed visibility graph averaging (VGA) operator is based on the visibility graph, hence, it conserves the time information likewise.

In some aggregation operator, how to decide the weight of each argument is a problem \cite{xu2005overview,filev1998issue,beliakov2005learning}, but in this proposed VGA operator, while the time series is converted into graphs, the degree distribution is decided, and meanwhile the weights are decided. In general, if the degree of a node is bigger than others, this node will be more important, and in the visibility graph a node represents a data value of time series, so it offers a reasonable way to determine the weights of the corresponding data value.

The remainder of this paper is organized as follows. Section \ref{Preliminaries} briefly introduces some necessary preliminaries of the aggregation operators and the graph theory. Section \ref{Visibility Graph Averaging Operator} details the proposed visibility graph averaging (VGA) aggregation operator. The properties of the visibility graph averaging (VGA) aggregation operator will be discussed in the Section \ref{properties}. In Section \ref{Case}, VGA aggregation operator is applied in the economic and compared with OWA operators to show its advantage. Finally, some conclusions are given in Section \ref{conclusion}.

\section{Preliminaries}
\label{Preliminaries}
In this section, the aggregation operators problem and the graph theory are briefly introduced.
\subsection{The aggregation operator}
Aggregating values, a new value can be obtained, but this can be done in different ways. In other words, aggregation operators is various. In the following, the aggregation operator will be introduced in a formal way.

Let $I$ be the unit intervals. ${A_j}(x) \in I$ denotes the degree to which $x$ satisfies the criteria $A_j$, $R(x) \in I$ denotes the set of the results of the aggregation.

\begin{definition}
An aggregation operator is a function $Agg$:
\[R(x) = Agg({A_1}(x),{A_2}(x), \cdots {A_n}(x))\]
where $n$ represents the number of values to be aggregated.
\end{definition}

Several fundamental conditions have been proposed to define the aggregation operators \cite{mayor1986representation}. The fundamental properties which generalize most of the precedent definitions are as follows \cite{detyniecki2001fundamentals}:

1)\textbf{Identity when unary}: If there is only one value needing to be aggregated, the result is itself. \[Agg({A_j}(x)) = {A_j}(x)\]

2)\textbf{Boundary conditions}: If all the values needing to be aggregated are completely bad, false or not satisfactory, the result has to be completely bad, false or not satisfactory. On the contrary, if all the values needing to be aggregated are completely good, true or satisfactory then the result has to be completely good, true or satisfactory.
\[Agg(0,0, \cdots ,0) = 0\]
\[Agg(1,1, \cdots ,1) = 1\]

3)\textbf{Monotonicity}: If the individual value increases the overall satisfaction should increase:
\[Agg({A_j}(x)) \ge Agg({A_j}(y)),{\kern 1pt} {\kern 1pt} {\kern 1pt} {\kern 1pt} {\kern 1pt} {\kern 1pt} {\kern 1pt} {\kern 1pt} {\kern 1pt} {\kern 1pt} {\kern 1pt} {\kern 1pt} {\kern 1pt} {\kern 1pt} if{\kern 1pt} {\kern 1pt} {\kern 1pt} {\kern 1pt} {A_j}(x) \ge {A_j}(y)\]

\subsection{The graph theory}
Graph theory is the study of graphs, which is made up of vertices and edges. Graphs can be used to deal with many types of relations and processes in computer science \cite{narsingh2004graph}, biological \cite{pavlopoulos2011using,haggarty2003chemical}, social \cite{nastos2013familial} and so forth \cite{kim2013predicting,agosta2013brain,liu2012description}. Recently, with the development of complex network research, the graph theory is widely used in the analysis of complex network \cite{barabasi1999emergence,gao2013modified,wei2013identifying}.

\begin{definition}
A graph is formed by vertices and edges connecting the vertices.
\end{definition}

\begin{example}
A graph is a pair of sets (V,E), where V denotes the set of vertices and E denotes the set of edges.
Figure \ref{graph} shows a graph with 5 vertices and 5 edges. The vertices are labeled as ${v_1},{v_2}, \cdots {v_5}$ and the edges are labeled as ${e_1},{e_2}, \cdots {e_5}$.
\end{example}

\begin{definition}
The degree $d(v)$ of a vertex $v$ is the number of edges at $v$.
\end{definition}

\begin{example}
In the Figure \ref{graph}, $d({v_1}) = 5,d({v_2}) = 2,d({v_3}) = 2,d({v_4}) = 1,d({v_5}) = 0$, the sum of the degree is $5+2+2+1=10$.
\end{example}

\begin{figure}[!t]
\centering
\includegraphics[scale=0.8]{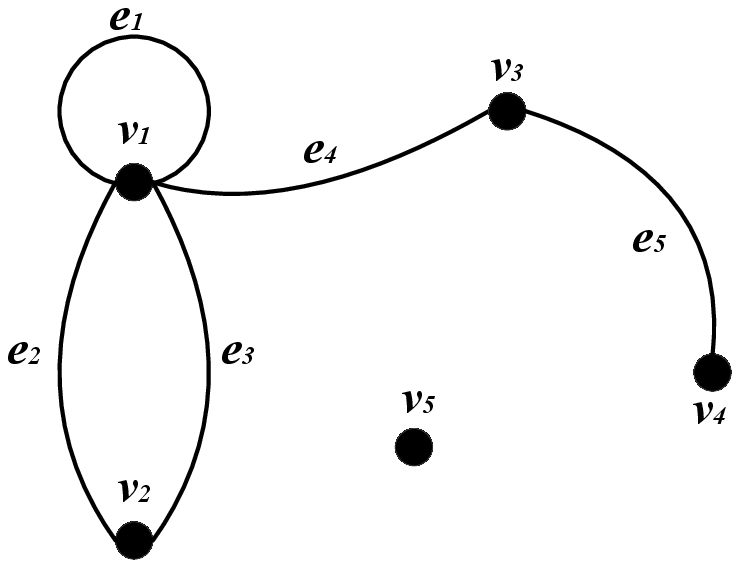}
\caption{The graph with 5 vertices and 5 edges}
\label{graph}
\end{figure}

In general, if a vertex has a big degree, it illustrates that this vertex links with many other vertices and it will be important.

\subsection{The visibility graph}
The visibility graph method was first proposed by L. Lacasa et al. in 2008, which can convert a time series into a graph. The properties of the time series is conserved in the graph topology. In the visibility graph, the values of time series are plotted by using vertical bars. A vertical bar links with others which can be seen from the top of itself.  The visibility criteria is established in the literature \cite{lacasa2008time} as follows:

\begin{definition}
Two data value $({t_1},{y_1})$ and $({t_2},{y_2})$ have visibility, if any other value $({t_3},{y_3})$ is placed between them fulfills:
\[{y_3} < {y_2} + ({y_1} - {y_2})\frac{{{t_2} - {t_3}}}{{{t_2} - {t_1}}}\]
\end{definition}

\begin{example}
In the Figure \ref{visibilityexample}, the histogram shows a time series with 11 data values, and according to the visibility algorithm, the associated graph is obtained. In the histogram, if a bar can be seen from the top of considered one, they will be linked. If two bar are linked in the histogram, the vertices which represent them will be linked in the associated graph.
\end{example}

\begin{figure}[!t]
\centering
\includegraphics[scale=0.5]{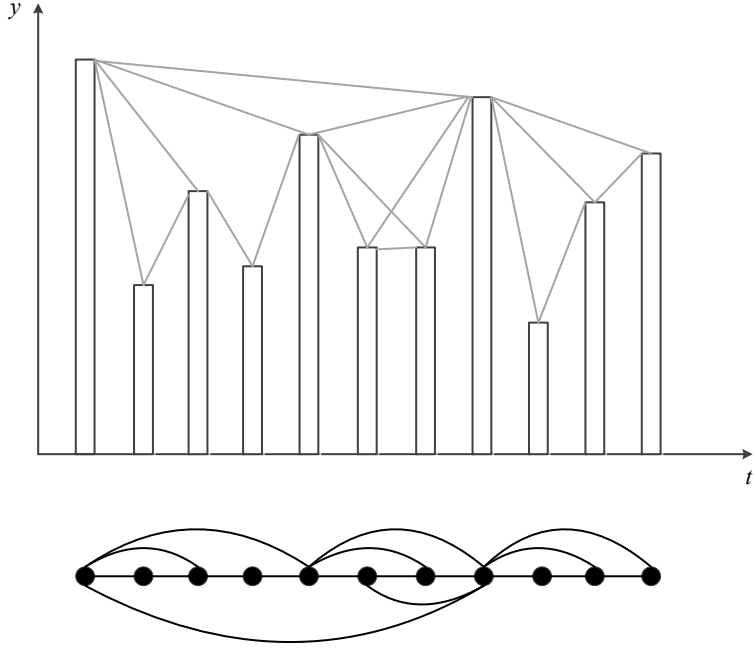}
\caption{The visibility graph}
\label{visibilityexample}
\end{figure}

The associated graph derived from a time series has the following properties:
1)\textbf{Connected}: a node can see its nearest neighbors.
2)\textbf{Undirected}: the associated graph extracted from a time series is undirected.
3)\textbf{Invariant under affine transformations of the series data}: Although rescale both horizontal and vertical axes, the visibility criterion is invariant.

\section{Visibility Graph Averaging Operator}
\label{Visibility Graph Averaging Operator}
As mentioned above, there is rarely aggregation operator considers the influence of time, but in many disciplines, time is a very important factor that can not be ignored. In this section, a new type of operator called visibility graph averaging (VGA) operator will be introduced which can be used to aggregate time series. This proposed VGA operator will be introduced in the following.

To a time series, if there are $i$ data values, the associated graph derived from the visibility algorithm will have $i$ vertices. VGA operator is defined as follows:
\begin{definition}
VGA operator is a mapping F:
\[{I^n} \to I,I \in R\]
where
\begin{equation}
\label{VGA}
F({a_1},{a_2}, \cdots {a_n}) = {w_1}{a_1} + {w_2}{a_2} +  \cdots  + {w_n}{a_n}
\end{equation}
and where $a_i$ is the $ith$ value in the time series, and $w_i$ is the weight of the value $a_i$ satisfies:

\[{w_i} \in [0,1]\]

\[\sum\limits_i^n {{w_i}}  = 1\]
\end{definition}

Because the degree distribution can reflect the importance of a vertex.
The weight $w_i$ can be obtained according to the degree distribution in the visibility graph which is defined as follows:
\begin{definition}
\begin{equation}
\label{weight}
{w_i} = \frac{{{d_i}}}{{\sum\limits_{i = 1}^n {{d_i}} }}
\end{equation}
where $d_i$ is the value of the degree at the vertex $i$.
\end{definition}
It is necessary to notice that if there is only one value in the time series, in the associated visibility graph, there just exist one vertex and the degree of this vertex is equal to zero. Hence, the weight is revised as one, then we have:
\begin{definition} If there is only one value in the time series, the VGA operator is as:
\[F({a_i}) = {a_i}\]
\end{definition}
The following simple example illustrates the use of this VGA operator.
\begin{example}
\label{example3.1}
Assume there is a time series which is shown in the Table \ref{example1}, and the associated visibility graph is shown in the Figure \ref{visibilityexample3.1}.
\end{example}

\begin{table}[htbp]
    \caption{Example of a time series(8 data values)}
    \label{example1}
    \begin{center}
    \begin{tabular}{cc|cc} \hline
    $T$ & $V$ & $T$ & $V$   \\ \hline
    $t_1$ & $40$ & $t_5$ & $85$  \\
    $t_2$ & $45$ & $t_6$ & $55$ \\
    $t_3$ & $70$ & $t_7$ & $70$  \\
    $t_4$ & $50$ & $t_8$ & $75$  \\ \hline
    \end{tabular}
    \end{center}
\end{table}

\begin{figure}[!t]
\centering
\includegraphics[scale=0.9]{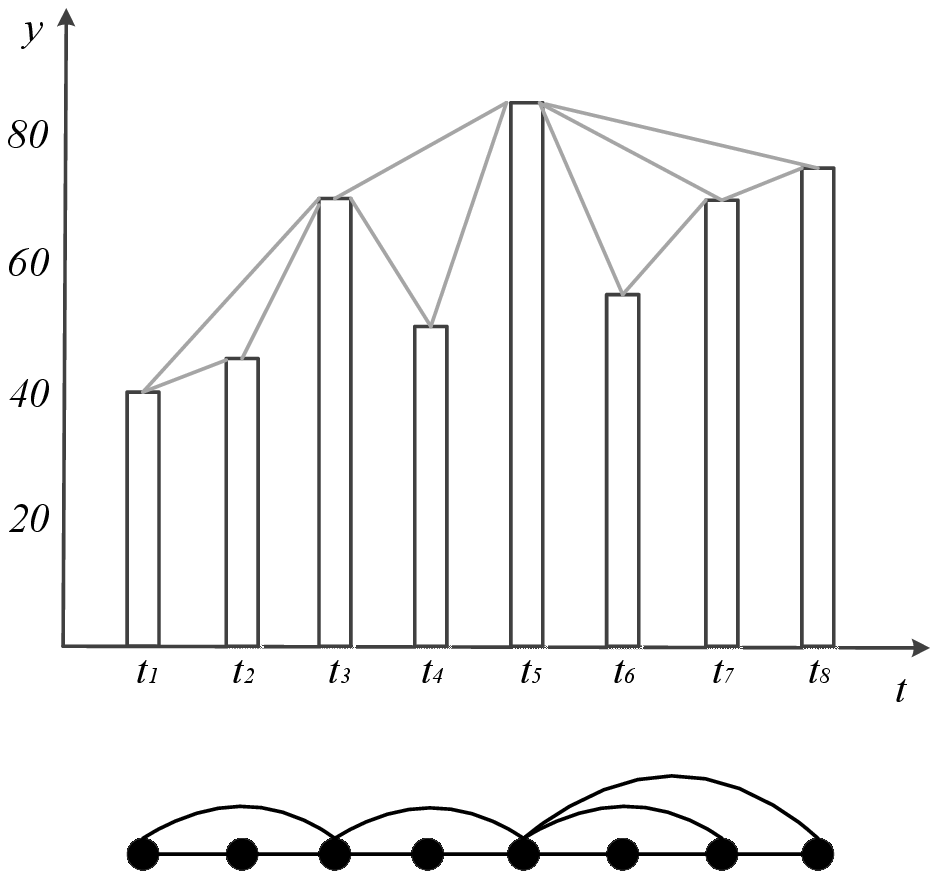}
\caption{The visibility graph of Example \ref{example3.1}}
\label{visibilityexample3.1}
\end{figure}

According to the Eq.\ref{weight}, the weights of the values can be computed as follows:
\[{w_1} = \frac{{{d_1}}}{{{d_1} + {d_2} +  \cdots  + {d_8}}}{\kern 1pt}  = \frac{2}{{2 + 2 + 4 + 2 + 5 + 2 + 3 + 2}} = \frac{2}{{22}} = \frac{1}{{11}};{\kern 1pt} {\kern 1pt} {\kern 1pt} {\kern 1pt} {\kern 1pt} {\kern 1pt} {\kern 1pt} {\kern 1pt} {\kern 1pt} {\kern 1pt} \]

%
%
%

So according to the Eq.\ref{VGA}, we have
\[\begin{array}{l}
 F(40,45,70,50,85,55,70,75) \\
 {\kern 1pt}  = \frac{1}{{11}} \times 40 + \frac{1}{{11}} \times 45 + \frac{2}{{11}} \times 70 + \frac{1}{{11}} \times 50 + \frac{5}{{22}} \times 85 + \frac{1}{{11}} \times 55 + \frac{3}{{22}} \times 70 + \frac{1}{{11}} \times 75 \\
 {\kern 1pt}  = 65.68 \\
 \end{array}\]

The weights of the VGA operator are obtained according to the distribution of the degree in the visibility graph, because in the network, if a node links with more nodes than others, its degree will bigger and it will be more important than others. In the histogram, Whether the considered one can see others is relevant to the order of the data. The order in our VGA operator is decided by time. The order of the data is not set artificially, and instead it is restricted by time. In other words, the time information is conserved in the VGA operator.

\section{Properties of VGA Operator}
\label{properties}
In this section, we shall investigate some properties of this new operator. These properties will be divided into two parts: the mathematical properties and the time properties.

\subsection{Mathematical property}

\begin{theorem}
Idempotency: if $a_i=a$,for all $i = 1,2, \cdots ,n$, then
\[F({a_1},{a_2}, \cdots ,{a_n}) = a\]
\end{theorem}

\begin{proof}
\[\begin{array}{l}
 F({a_1},{a_2}, \cdots ,{a_n}) = {w_1}a + {w_2}a +  \cdots  + {w_n}a \\
 {\kern 1pt} {\kern 1pt} {\kern 1pt} {\kern 1pt} {\kern 1pt} {\kern 1pt} {\kern 1pt} {\kern 1pt} {\kern 1pt} {\kern 1pt} {\kern 1pt} {\kern 1pt} {\kern 1pt} {\kern 1pt} {\kern 1pt} {\kern 1pt} {\kern 1pt} {\kern 1pt} {\kern 1pt} {\kern 1pt} {\kern 1pt} {\kern 1pt} {\kern 1pt} {\kern 1pt} {\kern 1pt} {\kern 1pt} {\kern 1pt} {\kern 1pt} {\kern 1pt} {\kern 1pt} {\kern 1pt} {\kern 1pt} {\kern 1pt} {\kern 1pt} {\kern 1pt} {\kern 1pt} {\kern 1pt} {\kern 1pt} {\kern 1pt} {\kern 1pt} {\kern 1pt} {\kern 1pt} {\kern 1pt} {\kern 1pt} {\kern 1pt} {\kern 1pt} {\kern 1pt} {\kern 1pt} {\kern 1pt} {\kern 1pt} {\kern 1pt} {\kern 1pt} {\kern 1pt} {\kern 1pt} {\kern 1pt} {\kern 1pt} {\kern 1pt} {\kern 1pt} {\kern 1pt} {\kern 1pt} {\kern 1pt} {\kern 1pt} {\kern 1pt} {\kern 1pt} {\kern 1pt} {\kern 1pt} {\kern 1pt} {\kern 1pt} {\kern 1pt} {\kern 1pt} {\kern 1pt} {\kern 1pt} {\kern 1pt} {\kern 1pt} {\kern 1pt} {\kern 1pt} {\kern 1pt}  = ({w_1} + {w_2} +  \cdots {w_n})a \\
 {\kern 1pt} {\kern 1pt} {\kern 1pt} {\kern 1pt} {\kern 1pt} {\kern 1pt} {\kern 1pt} {\kern 1pt} {\kern 1pt} {\kern 1pt} {\kern 1pt} {\kern 1pt} {\kern 1pt} {\kern 1pt} {\kern 1pt} {\kern 1pt} {\kern 1pt} {\kern 1pt} {\kern 1pt} {\kern 1pt} {\kern 1pt} {\kern 1pt} {\kern 1pt} {\kern 1pt} {\kern 1pt} {\kern 1pt} {\kern 1pt} {\kern 1pt} {\kern 1pt} {\kern 1pt} {\kern 1pt} {\kern 1pt} {\kern 1pt} {\kern 1pt} {\kern 1pt} {\kern 1pt} {\kern 1pt} {\kern 1pt} {\kern 1pt} {\kern 1pt} {\kern 1pt} {\kern 1pt} {\kern 1pt} {\kern 1pt} {\kern 1pt} {\kern 1pt} {\kern 1pt} {\kern 1pt} {\kern 1pt} {\kern 1pt} {\kern 1pt} {\kern 1pt} {\kern 1pt} {\kern 1pt} {\kern 1pt} {\kern 1pt} {\kern 1pt} {\kern 1pt} {\kern 1pt} {\kern 1pt} {\kern 1pt} {\kern 1pt} {\kern 1pt} {\kern 1pt} {\kern 1pt} {\kern 1pt} {\kern 1pt} {\kern 1pt} {\kern 1pt} {\kern 1pt} {\kern 1pt} {\kern 1pt} {\kern 1pt} {\kern 1pt} {\kern 1pt} {\kern 1pt} {\kern 1pt}  = 1 \times a = a \\
 \end{array}\]
\end{proof}

\begin{remark}
Particularly, there exist
\[F(0,0, \cdots ,0) = 0\]
\[F(1,1, \cdots ,1) = 1\]
The above equations illustrate that if every single value is small, the aggregated result will be small, and if every value is big, the aggregated will be big.
\end{remark}

\begin{theorem}
Stability for a linear function:
\[F(r{a_1} + t,r{a_2} + t, \cdots ,r{a_n} + t) = r \cdot F({a_1},{a_2}, \cdots ,{a_n}) + t\]
Where $r,t \in R$
\end{theorem}

\begin{proof}
As mentioned in the preliminaries, the visibility criterion is invariant under affine transformations, hence, the weight distribution of time series ${a_1},{a_2}, \cdots ,{a_n}$ is same as the weight distribution of time series $r{a_1} + t,r{a_2} + t, \cdots ,r{a_n} + t$. We have:
\[\begin{array}{l}
 F(r{a_1} + t,r{a_2} + t, \cdots ,r{a_n} + t) = {w_1}(r{a_1} + t) + {w_2}(r{a_2} + t) +  \cdots  + {w_n}(r{a_n} + t) \\
 {\kern 1pt} {\kern 1pt} {\kern 1pt} {\kern 1pt} {\kern 1pt} {\kern 1pt} {\kern 1pt} {\kern 1pt} {\kern 1pt} {\kern 1pt} {\kern 1pt} {\kern 1pt} {\kern 1pt} {\kern 1pt} {\kern 1pt} {\kern 1pt} {\kern 1pt} {\kern 1pt} {\kern 1pt} {\kern 1pt} {\kern 1pt} {\kern 1pt} {\kern 1pt} {\kern 1pt} {\kern 1pt} {\kern 1pt} {\kern 1pt} {\kern 1pt} {\kern 1pt} {\kern 1pt} {\kern 1pt} {\kern 1pt} {\kern 1pt} {\kern 1pt} {\kern 1pt} {\kern 1pt} {\kern 1pt} {\kern 1pt} {\kern 1pt} {\kern 1pt} {\kern 1pt} {\kern 1pt} {\kern 1pt} {\kern 1pt} {\kern 1pt} {\kern 1pt} {\kern 1pt} {\kern 1pt} {\kern 1pt} {\kern 1pt} {\kern 1pt} {\kern 1pt} {\kern 1pt} {\kern 1pt} {\kern 1pt} {\kern 1pt} {\kern 1pt} {\kern 1pt} {\kern 1pt} {\kern 1pt} {\kern 1pt} {\kern 1pt} {\kern 1pt} {\kern 1pt} {\kern 1pt} {\kern 1pt} {\kern 1pt} {\kern 1pt} {\kern 1pt} {\kern 1pt} {\kern 1pt} {\kern 1pt} {\kern 1pt} {\kern 1pt} {\kern 1pt} {\kern 1pt} {\kern 1pt} {\kern 1pt} {\kern 1pt} {\kern 1pt} {\kern 1pt} {\kern 1pt} {\kern 1pt} {\kern 1pt} {\kern 1pt} {\kern 1pt} {\kern 1pt} {\kern 1pt} {\kern 1pt} {\kern 1pt} {\kern 1pt} {\kern 1pt} {\kern 1pt} {\kern 1pt} {\kern 1pt} {\kern 1pt} {\kern 1pt} {\kern 1pt} {\kern 1pt} {\kern 1pt} {\kern 1pt} {\kern 1pt} {\kern 1pt} {\kern 1pt} {\kern 1pt} {\kern 1pt} {\kern 1pt} {\kern 1pt} {\kern 1pt} {\kern 1pt} {\kern 1pt} {\kern 1pt} {\kern 1pt} {\kern 1pt} {\kern 1pt} {\kern 1pt} {\kern 1pt} {\kern 1pt} {\kern 1pt} {\kern 1pt} {\kern 1pt} {\kern 1pt} {\kern 1pt} {\kern 1pt} {\kern 1pt} {\kern 1pt} {\kern 1pt} {\kern 1pt} {\kern 1pt} {\kern 1pt} {\kern 1pt} {\kern 1pt} {\kern 1pt} {\kern 1pt}  = {w_1}r{a_1} + {w_1}t + {w_2}r{a_2} + {w_2}t +  \cdots  + {w_n}r{a_n} + {w_n}t \\
 {\kern 1pt} {\kern 1pt} {\kern 1pt} {\kern 1pt} {\kern 1pt} {\kern 1pt} {\kern 1pt} {\kern 1pt} {\kern 1pt} {\kern 1pt} {\kern 1pt} {\kern 1pt} {\kern 1pt} {\kern 1pt} {\kern 1pt} {\kern 1pt} {\kern 1pt} {\kern 1pt} {\kern 1pt} {\kern 1pt} {\kern 1pt} {\kern 1pt} {\kern 1pt} {\kern 1pt} {\kern 1pt} {\kern 1pt} {\kern 1pt} {\kern 1pt} {\kern 1pt} {\kern 1pt} {\kern 1pt} {\kern 1pt} {\kern 1pt} {\kern 1pt} {\kern 1pt} {\kern 1pt} {\kern 1pt} {\kern 1pt} {\kern 1pt} {\kern 1pt} {\kern 1pt} {\kern 1pt} {\kern 1pt} {\kern 1pt} {\kern 1pt} {\kern 1pt} {\kern 1pt} {\kern 1pt} {\kern 1pt} {\kern 1pt} {\kern 1pt} {\kern 1pt} {\kern 1pt} {\kern 1pt} {\kern 1pt} {\kern 1pt} {\kern 1pt} {\kern 1pt} {\kern 1pt} {\kern 1pt} {\kern 1pt} {\kern 1pt} {\kern 1pt} {\kern 1pt} {\kern 1pt} {\kern 1pt} {\kern 1pt} {\kern 1pt} {\kern 1pt} {\kern 1pt} {\kern 1pt} {\kern 1pt} {\kern 1pt} {\kern 1pt} {\kern 1pt} {\kern 1pt} {\kern 1pt} {\kern 1pt} {\kern 1pt} {\kern 1pt} {\kern 1pt} {\kern 1pt} {\kern 1pt} {\kern 1pt} {\kern 1pt} {\kern 1pt} {\kern 1pt} {\kern 1pt} {\kern 1pt} {\kern 1pt} {\kern 1pt} {\kern 1pt} {\kern 1pt} {\kern 1pt} {\kern 1pt} {\kern 1pt} {\kern 1pt} {\kern 1pt} {\kern 1pt} {\kern 1pt} {\kern 1pt} {\kern 1pt} {\kern 1pt} {\kern 1pt} {\kern 1pt} {\kern 1pt} {\kern 1pt} {\kern 1pt} {\kern 1pt} {\kern 1pt} {\kern 1pt} {\kern 1pt} {\kern 1pt} {\kern 1pt} {\kern 1pt} {\kern 1pt} {\kern 1pt} {\kern 1pt} {\kern 1pt} {\kern 1pt} {\kern 1pt} {\kern 1pt} {\kern 1pt} {\kern 1pt} {\kern 1pt} {\kern 1pt} {\kern 1pt} {\kern 1pt} {\kern 1pt} {\kern 1pt} {\kern 1pt} {\kern 1pt} {\kern 1pt} {\kern 1pt}  = r({w_1}{a_1} + {w_2}{a_2} +  \cdots {w_n}{a_n}) + t({w_1} + {w_2} +  \cdots  + {w_n}) \\
 {\kern 1pt} {\kern 1pt} {\kern 1pt} {\kern 1pt} {\kern 1pt} {\kern 1pt} {\kern 1pt} {\kern 1pt} {\kern 1pt} {\kern 1pt} {\kern 1pt} {\kern 1pt} {\kern 1pt} {\kern 1pt} {\kern 1pt} {\kern 1pt} {\kern 1pt} {\kern 1pt} {\kern 1pt} {\kern 1pt} {\kern 1pt} {\kern 1pt} {\kern 1pt} {\kern 1pt} {\kern 1pt} {\kern 1pt} {\kern 1pt} {\kern 1pt} {\kern 1pt} {\kern 1pt} {\kern 1pt} {\kern 1pt} {\kern 1pt} {\kern 1pt} {\kern 1pt} {\kern 1pt} {\kern 1pt} {\kern 1pt} {\kern 1pt} {\kern 1pt} {\kern 1pt} {\kern 1pt} {\kern 1pt} {\kern 1pt} {\kern 1pt} {\kern 1pt} {\kern 1pt} {\kern 1pt} {\kern 1pt} {\kern 1pt} {\kern 1pt} {\kern 1pt} {\kern 1pt} {\kern 1pt} {\kern 1pt} {\kern 1pt} {\kern 1pt} {\kern 1pt} {\kern 1pt} {\kern 1pt} {\kern 1pt} {\kern 1pt} {\kern 1pt} {\kern 1pt} {\kern 1pt} {\kern 1pt} {\kern 1pt} {\kern 1pt} {\kern 1pt} {\kern 1pt} {\kern 1pt} {\kern 1pt} {\kern 1pt} {\kern 1pt} {\kern 1pt} {\kern 1pt} {\kern 1pt} {\kern 1pt} {\kern 1pt} {\kern 1pt} {\kern 1pt} {\kern 1pt} {\kern 1pt} {\kern 1pt} {\kern 1pt} {\kern 1pt} {\kern 1pt} {\kern 1pt} {\kern 1pt} {\kern 1pt} {\kern 1pt} {\kern 1pt} {\kern 1pt} {\kern 1pt} {\kern 1pt} {\kern 1pt} {\kern 1pt} {\kern 1pt} {\kern 1pt} {\kern 1pt} {\kern 1pt} {\kern 1pt} {\kern 1pt} {\kern 1pt} {\kern 1pt} {\kern 1pt} {\kern 1pt} {\kern 1pt} {\kern 1pt} {\kern 1pt} {\kern 1pt} {\kern 1pt} {\kern 1pt} {\kern 1pt} {\kern 1pt} {\kern 1pt} {\kern 1pt} {\kern 1pt} {\kern 1pt} {\kern 1pt} {\kern 1pt} {\kern 1pt} {\kern 1pt} {\kern 1pt} {\kern 1pt} {\kern 1pt} {\kern 1pt} {\kern 1pt} {\kern 1pt} {\kern 1pt} {\kern 1pt} {\kern 1pt} {\kern 1pt} {\kern 1pt}  = r \cdot F({a_1},{a_2}, \cdots ,{a_n}) + t \\
 \end{array}\]

\end{proof}

\subsection{Time property}

Because the visibility graph averaging (VGA) operator decides the weights according to the degree distribution of the associated graph derived from the visibility algorithm, and the visibility graph conserves the structure of the time series, certainly the VGA will conserve the structure of the time series likewise.

It is obvious that if a series is periodic, the weights of data in this series will be periodic except for weights of the first period and the last period. The Figure \ref{periodic} shows the weights of a periodic series $({\rm{4 ,3, 2, 5, 1, }} \cdots {\rm{4, 3, 2, 5, 1}})$. If a series is random, the weights will be random. If the series is fractal, the weights will be scale-free. The Figure \ref{fractal} shows the weights distribution of the Conway series.

\begin{figure}[!t]
\centering
\includegraphics[scale=0.7]{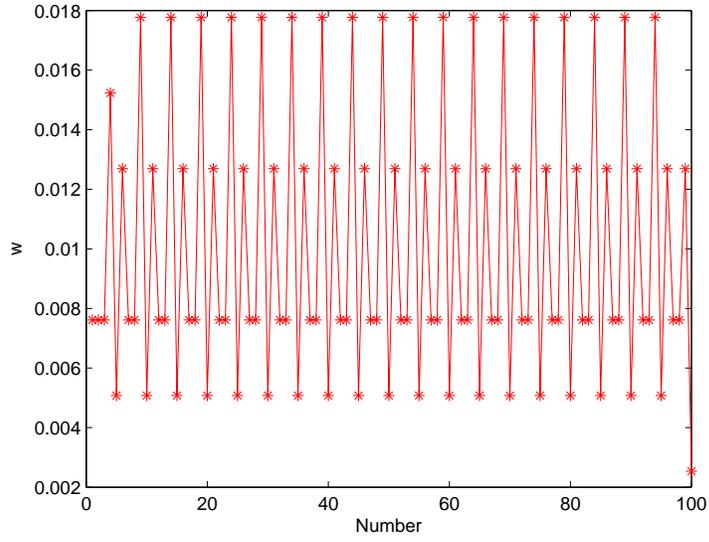}
\caption{The weights of a periodic series}
\label{periodic}
\end{figure}

\begin{figure}[!t]
\centering
\includegraphics[scale=0.7]{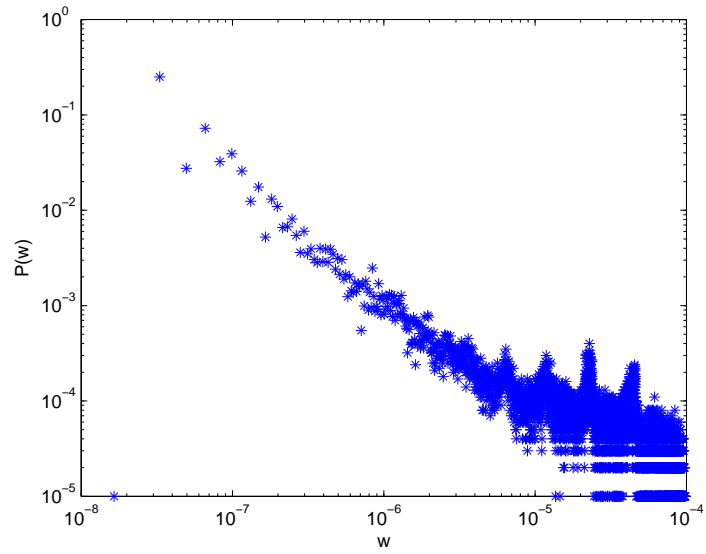}
\caption{The weights of a fractal series}
\label{fractal}
\end{figure}

Here we consider some special cases to show its special time properties. Suppose that there is a periodic series with one special big value which is like the Figure \ref{property1} showing. In this figure, if a data value is bigger than others, it will be more important (In the VGA operator, its weight will be big), because it can see more other bars in the histogram, but with time passing, the number of bars that it can see will get smaller (In the VGA operator, its weight will get smaller). In the figure, the periodic series on the left and right sides of the special big value are divided into 4 groups. In the group 1, in total $8$ values can be seen from the top of the special one. In the group 2, there are 6 values being seen. In the group 3, there are 4 values, but in the group, there are just 2 values. In other words, time will decrease the influence of a bigger value. This property reflects the reality. In the real world, some important events will influence other events, but with time passing, this influence will be decreased (In the VGA operator, the weights will be decreased), and the number of the events which are influenced by the important event will be decreased too.

Moreover, if one argument $a$ is far away from the considered one argument $b$, the probability that $a$ can be seen from the top of $b$ will decrease, because with other data appearing between them, these data may block the sight. With the time passing, the distance between $a$ and $b$ will get bigger and bigger, then more and more other data will block between them, so the probability that they can see each other will get smaller. It means that time will decrease the probability of the connection between data. This situation is shown in the figure \ref{property2}.

In addition, if an argument with big value lies in between two arguments which are just a little smaller than it, it can not connect many other arguments and its weight will not be great. It illustrates that at what time the argument with big value appears in the histogram is important. All of the time properties will influence the weights distribution in the VGA operator.

\begin{figure}[!t]
\centering
\includegraphics[scale=0.5]{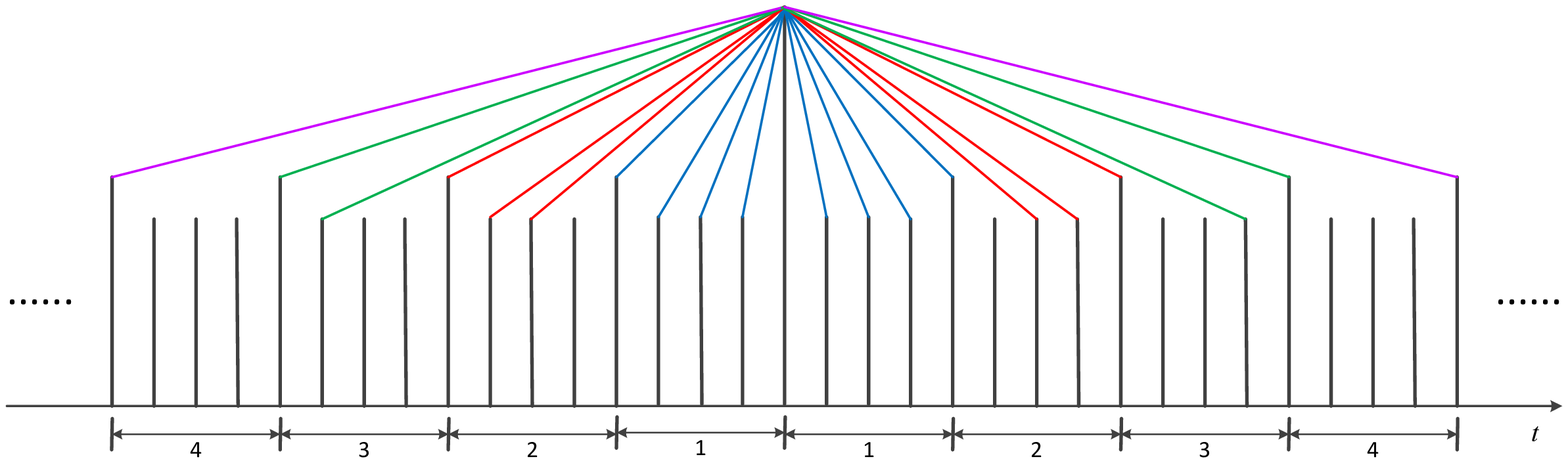}
\caption{The time property of VGA operator}
\label{property1}
\end{figure}

\begin{figure}[!t]
\centering
\includegraphics[scale=0.5]{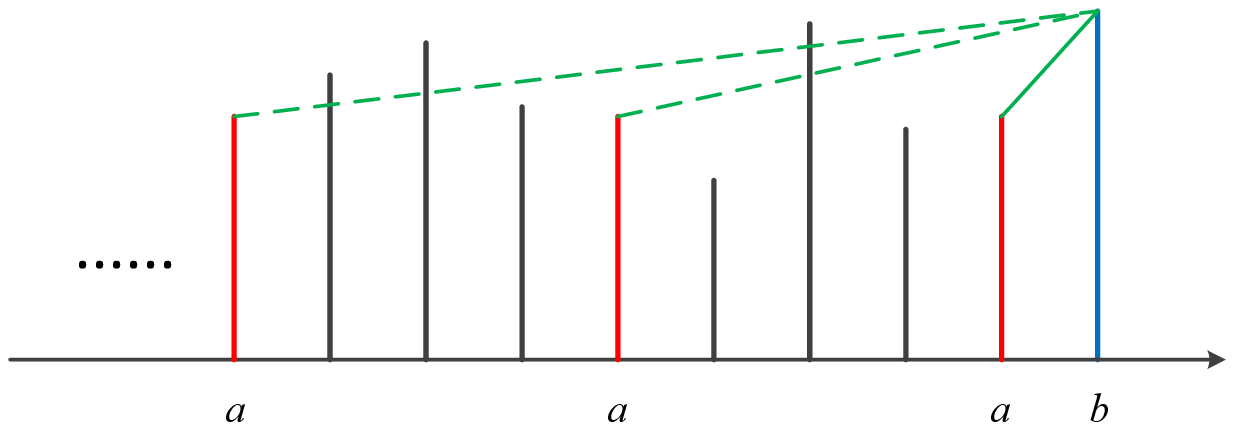}
\caption{The time property of VGA operator}
\label{property2}
\end{figure}

\begin{figure}[!t]
\centering
\includegraphics[scale=0.5]{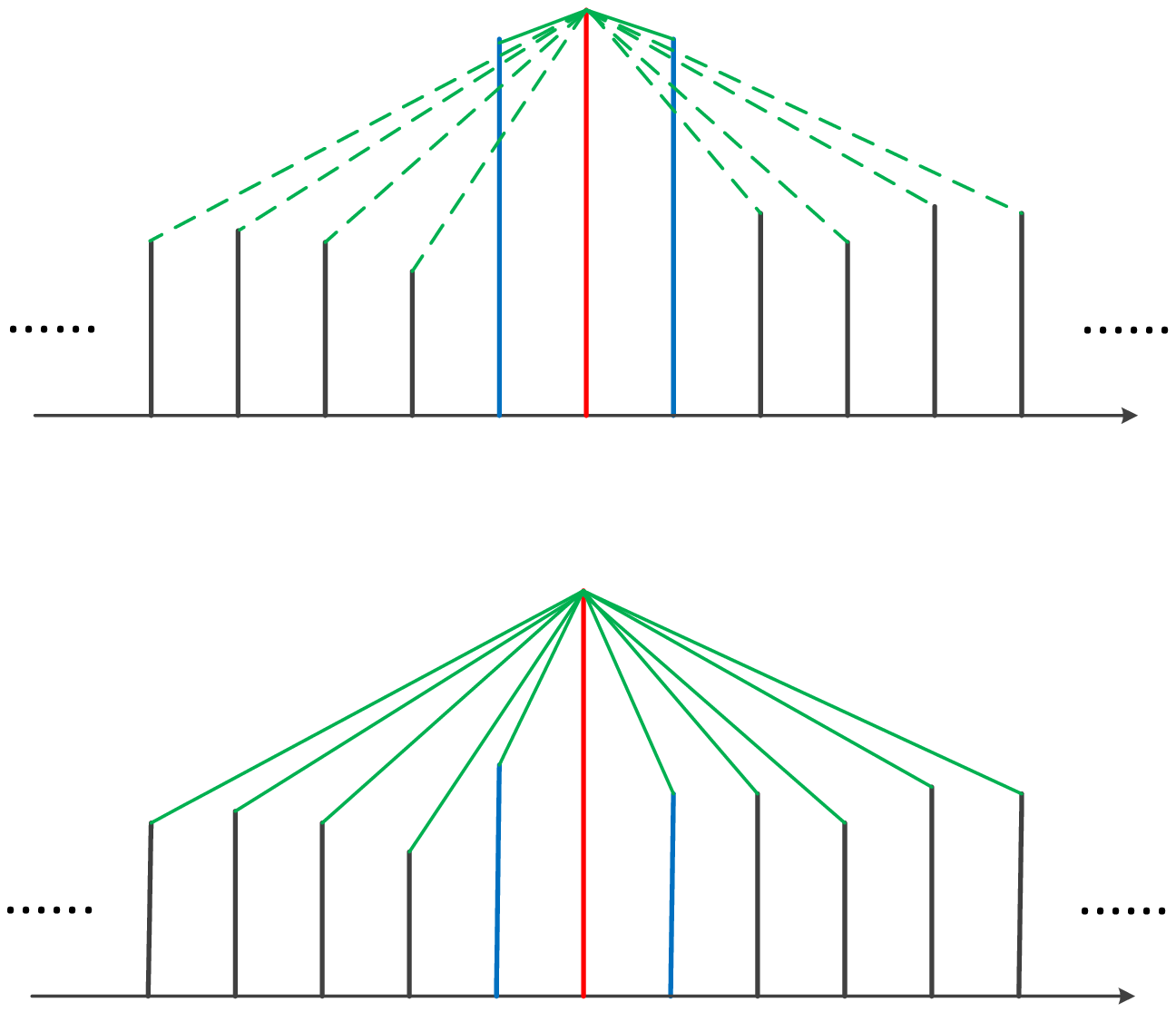}
\caption{The time property of VGA operator}
\label{property2}
\end{figure}

\section{Case Study}
\label{Case}
In this section, a case is used to illustrate the feasibility and the efficiency of the VGA operator. The VGA operator will be used in the analysis of the Taiwan Stock Exchange Capitalization Weighted Stock Index (TAIEX) database. In economic, forecasting can provide a superior investment strategy for investors. The problem is that it is difficult to handle high-order stock data, because it is hard to determine an appropriate weight to each period of past stock price. VGA operator will solve this problem conveniently.

The Taiwan Stock Exchange Capitalization Weighted Stock Index (TAIEX) of the year 2000-2012 is chosen as the forecasting basic data. The data of December from 2000 to 2012 which are listed in the Table \ref{basicdata} will be aggregated by VGA operator and OWA operators, respectively.

\begin{landscape}
\begin{table}[h]
\tiny
    \caption{The experimental data}
    \label{basicdata}
    \begin{center}
    \begin{tabular}{cccccccccccccc} \hline
     & $2000$ & $2001$ & $2002$ & $2003$ & $2004$ & $2005$ & $2006$ & $2007$ & $2008$ & $2009$ & $2010$ & $2011$ & $2012$  \\ \hline
    $12.1$ &5,342.06 & & &5,870.17 &5,798.62 & 6,179.82&7,613.57 & & 4,518.43& 7,649.23&8,520.11 & 7,178.69& \\
    $12.2$ &5,277.35 &  & 4,683.18 & 5,911.45 & 5,867.95 &6,228.95 &  &  & 4,356.98 &7,677.62 &8,585.77 &7,140.68 & \\
    $12.3$ & &4,646.61 & 4,793.93&5,884.97 &5,893.27 & & &8,583.84 & 4,307.26&7,684.67 &8,624.01 & & 7,599.91 \\
    $12.4$ &5,174.02 &4,766.43 &4,727.49 &5,920.46 & & &7,647.01 & 8,651.28& 4,254.96& 7,650.91& & &7,600.98 \\
    $12.5$ &5,199.20 & 4,924.56&4,755.40 &5,900.05 & & 6,348.31& 7,609.90& 8,676.95&4,225.07 & & & 7,098.08& 7,649.05\\
    $12.6$ &5,170.62 &5,208.86 &4,738.98 & &5,919.17 & 6,350.52&7,693.33 &8,694.41 & & &8,702.23 & 6,956.28& 7,623.26\\
    $12.7$ &5,212.73 &5,333.93 & & &5,925.28 &6,329.52 &7,686.52 &8,722.38 & &7,775.64 &8,704.39 &7,033.00 & 7,642.26\\
    $12.8$ &5,252.83 & & &5,847.15 &5,892.51 & 6,249.19& 7,636.30& &4,418.33 &7,768.71 &8,703.79 & 6,982.90& \\
    $12.9$ & & & 4,823.67& 5,859.56&5,913.97 &6,264.36 & & & 4,472.66&7,797.42 & 8,753.84& 6,893.30& \\
    $12.10$ & &5,321.28 &4,755.01 &5,803.42 &5,911.63 & & & 8,598.03& 4,658.87& 7,677.91& 8,718.83& & 7,609.50\\
    $12.11$ &5,284.41 &5,273.97 &4,699.41 & & & &7,612.12 &8,638.33 & 4,655.57& 7,795.07& & &7,613.69 \\
    $12.12$ &5,380.09 &5,539.31 & 4,669.70&5,858.32 & &6,266.29 &7,458.56 & 8,490.84& 4,481.27& & & 6,949.04 & 7,690.19\\
    $12.13$ & 5,384.36& 5,407.54& 4,588.14& &5,878.89 & 6,261.18&7,450.30 & 8,187.95& & & 8,736.59&6,896.31 & 7,757.09\\
    $12.14$ &5,320.16 &5,486.73 & & &5,909.65 & 6,235.35& 7,480.41&8,118.08 & &7,819.13 &8,740.43 & 6,922.57 &7,698.77 \\
    $12.15$ &5,224.74 & & &5,924.24 &6,002.58 &6,258.47 & 7,538.82& & 4,613.72& &8,756.71 & 6,764.59& \\
    $12.16$ &5,134.10 & &4,582.05 &5,887.23 &6,019.23 &6,350.69 & & &4,616.89 & 7,751.60& 8,782.20& 6,785.09 & \\
    $12.17$ & & 5,456.15& 4,545.62& 5,752.01&6,009.32 & & & 7,830.85&4,648.02 &7,742.17 & 8,817.90& & 7,631.28\\
    $12.18$ &5,055.20 &5,329.19 &4,535.93 & 5,768.76& & & 7,624.62&7,807.39 & 4,694.81&7,753.63 & & & 7,643.74\\
    $12.19$ & 5,040.25& 5,221.96& 4,549.23& 5,759.23& & 6,431.42& 7,598.88&8,014.31 &4,694.52 & & & 6,633.33 & 7,677.47\\
    $12.20$ &4,947.89 & 5,309.10&4,595.67 & & 5,985.94&6,427.84 & 7,648.35&7,857.08 & & & 8,768.72&6,662.64 & 7,595.46\\
    $12.21$ &4,817.22 & 5,109.24& & & 5,987.85&6,471.89 & 7,620.94&7,941.44 & & 7,787.27& 8,827.79& 6,966.48 &7,519.93\\
    $12.22$ &4,811.22 & & & 5,835.11& 6,001.52& 6,417.20& 7,652.47& &4,535.54 &7,856.00 & 8,860.49& 6,966.35 &7,540.14\\
    $12.23$ & & &4,572.77 &5,845.51 &5,997.67 &6,512.63 & & & 4,405.86&7,901.50 & 8,898.87& 7,110.73&\\
    $12.24$ & &5,164.73 &4,544.50 & 5,857.87& 6,019.42& & & 8,135.48& 4,423.09& 7,963.54& 8,861.10& &7,535.52\\
    $12.25$ & & 5,372.81& 4,484.43& 5,853.70& & & 7,646.81& 8,167.07& 4,413.45& 7,972.59& & &7,636.57\\
    $12.26$ & 4,721.36& 5,392.43&4,567.37 & 5,857.21& & 6,534.77&7,727.59 &8,156.39 &4,425.08 & & & 7,092.58 &7,634.19\\
    $12.27$ &4,614.63 & 5,332.98& 4,547.32& & 5,985.94&6,531.59 & 7,733.18& 8,313.72& & & 8,892.31& 7,085.03 &7,648.41\\
    $12.28$ & 4,797.14& 5,398.28& & &6,000.57 & 6,524.40& 7,732.93&8,396.95 & & 8,057.49& 8,870.76& 7,056.67 &7,699.50\\
    $12.29$ &4,743.94 & & & 5,804.89& 6,088.49& 6,575.53& 7,823.72& & 4,416.16& 8,053.83& 8,866.35& &\\
    $12.30$ & 4,739.09& & 4,457.75& 5,866.75& 6,100.86& 6,548.34& & & 4,589.04& 8,112.28& 8,907.91& 7,072.08 &\\
    $12.31$ & & 5,551.24& 4,452.45& 5,890.69& 6,139.69& & & 8,506.28& 4,591.22& 8,188.11& 8,972.50& &\\
     \hline
    \end{tabular}
    \end{center}
\end{table}
\end{landscape}

\subsection{Aggregating with OWA operators}

In this paper, the weights of the OWA operators can be identified by using Fuller and Majlender's approach. Fuller and Majlender transform Yager's OWA equation by using Lagrange multipliers. The main calculation process is as follows:

(1)If $n=2$, then ${w_1} = \alpha ,{w_2} = 1 - \alpha $.

(2)If $\alpha  = 0$ or $\alpha  = 1$, then $w = {(0,0, \cdots ,1)^T}$ or $w = {(1,0, \cdots ,0)^T}$, respectively.

(3)If $n \ge 3$ and $0 < \alpha  < 1$ then
\begin{equation}
{w_j} = \sqrt[{n - 1}]{{w_1^{n - j}w_n^{j - 1}}}
\end{equation}

\begin{equation}
{w_n} = \frac{{((n - 1)\alpha  - n){w_1} + 1}}{{(n - 1)\alpha  + 1 - n{w_1}}}
\end{equation}

\begin{equation}
{w_1}{[(n - 1)\alpha  + 1 - n{w_1}]^n} = {((n - 1)\alpha )^{n - 1}}[((n - 1)\alpha  - n){w_1} + 1]
\end{equation}
Where $\alpha $ characterizes the degree to which the aggregation is like an \emph{or} operation.

In the Table \ref{basicdata}, some data do not exist. For example, to December 1, the data of the year 2001, 2002, 2007 and 2012 do not exist, so the number of data which need to be aggregated is 9. However, to December 5, the number of data which need to be aggregated is 10. Hence, in the OWA operator, the parameter $n$ is different. In this case, when $n=8$, $n=9$ and $n=10$, the corresponding weights need to be computed which are shown in the Table \ref{n8}, Table \ref{n9} and Table \ref{n10}.

\begin{table}[htbp]
    \caption{OWA operators' weights when n=8}
    \label{n8}
    \begin{center}
    \begin{tabular}{ccccccccc} \hline
     & $w_1$ & $w_2$ & $w_3$ & $w_4$ & $w_5$ & $w_6$ & $w_7$ & $w_8$ \\ \hline
    $\alpha {\rm{ = }}0.1$ &0.0012  &0.0030 &0.0071 &0.0173 &0.0417 &0.1006 &0.2421 &0.5864  \\
    $\alpha {\rm{ = }}0.5$ &0.1250  &0.1250 &0.1250 &0.1250 &0.1250 &0.1250 &0.1250 &0.1250  \\
    $\alpha {\rm{ = }}0.6$ & 0.1917 & 0.1674 & 0.1461 & 0.1275 & 0.1113 &0.0972 &0.0848& 0.0740 \\
    $\alpha {\rm{ = }}0.9$ &0.5864  &0.2421  &0.1006  &0.0417&0.0173&0.0071&0.0030&0.0012  \\\hline
    \end{tabular}
    \end{center}
\end{table}

\begin{table}
\small
    \caption{OWA operators' weights when n=9}
    \label{n9}
    \begin{center}
    \begin{tabular}{cccccccccc} \hline
     & $w_1$ & $w_2$ & $w_3$ & $w_4$ & $w_5$ & $w_6$ & $w_7$ & $w_8$ & $w_9$ \\ \hline
     $\alpha {\rm{ = }}0.1$ &0.0009  &0.0020 &0.0044 &0.0098 &0.0220 &0.0493 &0.1104 &0.2473&0.5540  \\
     $\alpha {\rm{ = }}0.5$ &0.1111  &0.1111 &0.1111 &0.1111 &0.1111 &0.1111 &0.1111 &0.1111&0.1111  \\
    $\alpha {\rm{ = }}0.6$ & 0.1726 & 0.1527 & 0.1350 & 0.1195 & 0.1057 &0.0936 &0.0828& 0.0732& 0.0648 \\
    $\alpha {\rm{ = }}0.9$ &0.5540 &0.2473  &0.1104  &0.0493&0.0220&0.0098&0.0044&0.0020&0.0009  \\\hline
    \end{tabular}
    \end{center}
\end{table}

\begin{table}
\small
    \caption{OWA oprators' weights when n=10}
    \label{n10}
    \begin{center}
    \begin{tabular}{p{0.9cm}p{0.9cm}p{0.9cm}p{0.9cm}p{0.9cm}p{0.9cm}p{0.9cm}p{0.9cm}p{0.9cm}p{0.9cm}p{0.9cm}} \hline
     & $w_1$ & $w_2$ & $w_3$ & $w_4$ & $w_5$ & $w_6$ & $w_7$ & $w_8$ & $w_9$& $w_{10}$\\ \hline
     $\alpha {\rm{ = }}0.1$ &0.0007  &0.0014  &0.0029  &0.0061&0.0127&0.0268&0.0564&0.1186&0.2495&0.5250  \\
     $\alpha {\rm{ = }}0.5$ &0.1000  &0.1000  &0.1000  &0.1000&0.1000&0.1000&0.1000&0.1000&0.1000&0.1000  \\
    $\alpha {\rm{ = }}0.6$ & 0.1569 & 0.1404 & 0.1256 & 0.1123 & 0.1005 &0.0899 &0.0804& 0.0720&0.0644&0.0576 \\
    $\alpha {\rm{ = }}0.9$ &0.5250  &0.2495  &0.1186  &0.0564&0.0268&0.0127&0.0061&0.0029&0.0014&0.0007  \\\hline
    \end{tabular}
    \end{center}
\end{table}

The thirteen periods of stock prices, $P(2000),P(2001), \cdots ,P(2012)$ will be aggregated. According to the OWA operators, the aggregated value is computed as follows:
\begin{equation}
Agg = {w_1}{P_1} + {w_2}{P_2} +  \cdots  + {w_n}{P_n}
\end{equation}
Where $P_i$ is the $ith$ largest element of the price.
The aggregated results are listed in the Table \ref{OWAagg}.

\begin{table}[htbp]
\tiny
    \caption{Aggregated results obtained by OWA operators}
    \label{OWAagg}

    \begin{center}
    \begin{tabular}{cccccc} \hline
    Time & $\alpha {\rm{ = }}0.1$ & $\alpha {\rm{ = }}0.5$ & $\alpha {\rm{ = }}0.6$ & $\alpha {\rm{ = }}0.9$ \\ \hline
    $12.1$ &5016.64&6518.31&6890.91&8039.99 \\
    $12.2$ &4689.42&6191.60&6598.81&7974.68\\
    $12.3$ &4617.58&6445.85&6934.81&8345.14\\
    $12.4$ &4569.12&6265.32&6728.68&8119.05\\
    $12.5$ &4602.82&6238.66&6678.87&8067.38 \\
    $12.6$ &5075.26&6705.77&7135.36&8409.16 \\
    $12.7$ &5500.49&7036.57&7410.04&8467.75\\
    $12.8$ &4945.69&6527.31&6920.95&8165.84 \\
    $12.9$ &4817.88&6347.35&6757.01&8139.45\\
    $12.10$ &4893.96&6560.95&7022.27&8403.36 \\
    $12.11$ &4838.17&6446.57&6887.57&8184.62 \\
    $12.12$ &4785.57&6278.36&6677.09&7954.31 \\
    $12.13$ &5065.98&6654.83&7060.31&8296.42\\
    $12.14$ &5582.23&6973.13&7318.53&8330.72 \\
    $12.15$ &5006.11&6385.48&6755.32&8068.70 \\
    $12.16$ &4785.18&6211.50&6621.63&8070.54 \\
    $12.17$ &4818.32&6491.94&6935.00&8278.61\\
    $12.18$ &4752.56&6245.29&6639.31&7677.38 \\
    $12.19$ &4766.81&6162.06&6545.92&7700.12 \\
    $12.20$ &4957.02&6579.87&6986.89&8221.64 \\
    $12.21$ &5247.54&6905.01&7273.98&8326.75 \\
    $12.22$ &4957.93&6647.60&7051.87&8276.31\\
    $12.23$ &4726.36&6405.69&6847.39&8282.43 \\
    $12.24$ &4698.79&6499.93&6978.75&8397.26\\
    $12.25$ &4678.45&6443.43&6871.73&7965.21 \\
    $12.26$ &4648.71&6210.90&6625.86&7821.17\\
    $12.27$ &4860.76&6668.51&7116.24&8425.71\\
    $12.28$ &5320.42&7053.47&7439.94&8508.03 \\
    $12.29$ &4780.11&6546.61&7000.28&8376.40 \\
    $12.30$ &4670.64&6265.38&6724.67&8270.49 \\
    $12.31$ &4727.19&6536.52&7037.10&8561.03\\\hline
    \end{tabular}
    \end{center}
\end{table}

\subsection{Aggregating with VGA operator}
According to the visibility graph and Eq.\ref{weight}, the weights of the VGA operator can be obtained which are shown in the Table \ref{VGAweights}. The aggregated results are shown in the Table \ref{VGAagg}.

\begin{landscape}
\begin{table}[h]
\tiny
    \caption{The weights obtained by VGA operator}
    \label{VGAweights}
    \begin{center}
    \begin{tabular}{cccccccccccccc} \hline
     & $2000$ & $2001$ & $2002$ & $2003$ & $2004$ & $2005$ & $2006$ & $2007$ & $2008$ & $2009$ & $2010$ & $2011$ & $2012$  \\ \hline
    $12.1$ & 0.0714&&&0.1429&0.1071&0.1071&0.2500&&0.0714&0.1071&0.1071&0.0357 \\
    $12.2$ &0.1111  &   &  0.0556   & 0.1667  &  0.1111 &   0.1389    &           &     &  0.0556  &  0.1667 &   0.1667  &  0.0278 \\
    $12.3$ &  &  0.1154 &   0.0769  &  0.1538  &  0.0769   &         &        &  0.2308  &  0.0769  &  0.1154  &  0.1154   &       &  0.0385 \\
    $12.4$ &0.1471  &  0.1176 &   0.0882 &   0.1471    &        &        &  0.0882   & 0.2059  &  0.0588 &   0.0882   &         &       &   0.0588 \\
    $12.5$ &0.1250  &  0.1250  &  0.0750  &  0.1500     &      &  0.0500  &  0.1250  &  0.1750  &  0.0500     &         &       & 0.0750  &  0.0500\\
    $12.6$ &0.1136  &  0.1364  &  0.0909     &      &  0.1364   & 0.0909   & 0.1364  &  0.1364     &       &       &   0.0682   & 0.0455 &   0.0455\\
    $12.7$ &0.1250  &  0.1250   &          &       &  0.1250  &  0.1000  &  0.1250 &   0.1500    &       &  0.0500 &   0.1000  &  0.0500 &   0.0500\\
    $12.8$ &0.1000   &          &       &  0.1333 &   0.1000  &  0.1333  &  0.2333    &       &  0.0667  &  0.1000  &  0.1000  &  0.0333   &       \\
    $12.9$ &    &       &  0.0357  &  0.1786  &  0.1429 &   0.1786  &            &      &  0.0714  &  0.1786  &  0.1786  &  0.0357    &      \\
    $12.10$ & &  0.1154  &  0.0769 &   0.1538  &  0.0769   &          &       &  0.2308  &  0.0769  &  0.1154  &  0.1154   &        &  0.0385\\
    $12.11$ &0.1071  &  0.1429  &  0.1071    &           &         &   &   0.1429  &  0.2500 &   0.0714  &  0.1071    &       &       &   0.0714 \\
    $12.12$ &0.1190  &  0.1429   & 0.0476  &  0.1429   &         & 0.0952  &  0.1190   & 0.1667  &  0.0476    &        &        &  0.0714  &  0.0476\\
    $12.13$ &0.1190  &  0.1429  &  0.0952     &       & 0.1190   & 0.0952 &   0.1429 &   0.1190     &         &      & 0.0714 &   0.0476  &  0.0476\\
    $12.14$ &0.1111 &   0.1389      &        &     &  0.0833  &  0.1111  &  0.1389  &  0.1389     &       & 0.0556  &  0.1111  &  0.0556  &  0.0556 \\
    $12.15$ &0.0833   &           &     &  0.1667  &  0.1250  &  0.1250  &  0.2500    &       &  0.0833     &      &  0.1250  &  0.0417     &     \\
    $12.16$  &0.1111    &        & 0.0556  &  0.1667  &  0.1111 &   0.1389  &             &    &  0.0556  &  0.1667  &  0.1667 &   0.0278   &      \\
    $12.17$  & &   0.1429  &  0.0714  &  0.1429   & 0.1071      &        &      &  0.2143  &  0.0714  &  0.1071   & 0.1071   &        &  0.0357\\
    $12.18$ &0.0833   & 0.1667  &  0.0833  &  0.1250    &         &        & 0.1667  &  0.1250  &  0.0833  &  0.1250   &         &       &  0.0417\\
    $12.19$ & 0.1053  &  0.1316  &  0.0526 &   0.1316   &        &  0.1053 &   0.1316 &   0.1053   & 0.0789   &        &       &  0.0789  &  0.0789\\
    $12.20$ &0.0588   & 0.1471  &  0.0882   &        & 0.1176   & 0.0882  &  0.2059  &  0.0588       &       &     &   0.1176  &  0.0588  &  0.0588\\
    $12.21$ &0.1176   & 0.1176    &          &    &  0.1176  &  0.1176  &  0.1471 &   0.0882    &      &  0.0588  &  0.1176  &  0.0588  &  0.0588\\
    $12.22$ &0.0625    &         &    &   0.1250  &  0.0938   & 0.0938  &  0.2188    &       &  0.0625  &  0.0938   & 0.1250   & 0.0625  &  0.0625\\
    $12.23$ &      &     &  0.0417  &  0.2083  &  0.0833  &  0.2083   &           &       & 0.0833  &  0.1667 &   0.1667&    0.0417   &      \\
    $12.24$ &   &  0.1154 &   0.0769 &   0.1538  &  0.0769   &          &     &  0.2308   & 0.0769  &  0.1154   & 0.1154   &       &  0.0385\\
    $12.25$ &   & 0.1818  &  0.0909  &  0.1364    &           &     &  0.1364  &  0.1818 &   0.0909  &  0.1364    &         &       &  0.0455\\
    $12.26$ & 0.0333   & 0.1667  &  0.0667  &  0.1333    &      &  0.1000  &  0.1333  &  0.1333  &  0.0667    &        &      &  0.1000  &  0.0667\\
    $12.27$ &0.0313  &  0.1875   & 0.0938   &         & 0.1250  &  0.0938  &  0.1563  &  0.0938    &          &     &   0.0938  &  0.0625  &  0.0625\\
    $12.28$ & 0.0333   & 0.1667      &        &       & 0.1000  &  0.1000  &  0.1333  &  0.1333   &        & 0.0667   & 0.1333  &  0.0667  &  0.0667\\
    $12.29$ & 0.1071   &          &      &  0.1429  &  0.1071  &  0.1429  &  0.2500   &        &  0.0714  &  0.1071   & 0.0714   &           &    \\
    $12.30$ & 0.0938    &       &  0.0625  &  0.1875  &  0.0938   & 0.1563      &        &     &  0.0625 &   0.1250  &  0.1875  &  0.0313  &       \\
    $12.31$ &    & 0.1538  &  0.0769  &  0.1538  &  0.1154      &       &       &  0.2308 &   0.0769  &  0.1154  &  0.0769     &        &     \\
     \hline
    \end{tabular}
    \end{center}
\end{table}
\end{landscape}

\begin{table}[htbp]
    \caption{Aggregated results obtained by VGA operator}
    \label{VGAagg}

    \begin{center}
    \begin{tabular}{cc|cc} \hline
    Time & Aggregated Results & Time & Aggregated Results  \\ \hline
    $12.1$ & 6718.52 & 12.17 & 6626.59 \\
    $12.2$ & 6499.89 & 12.18 & 6334.17 \\
    $12.3$ & 6749.91 & 12.19 & 6235.76 \\
    $12.4$ & 6437.80 & 12.20 & 6655.86 \\
    $12.5$ & 6420.31 & 12.21 & 6803.89 \\
    $12.6$ & 6606.90 & 12.22 & 6902.51 \\
    $12.7$ & 6954.10 & 12.23 & 6728.45 \\
    $12.8$ & 6683.75 & 12.24 & 6758.54 \\
    $12.9$ & 6703.36 & 12.25 & 6545.96 \\
    $12.10$ & 6854.48 & 12.26 & 6426.12 \\
    $12.11$ & 6781.71 & 12.27 & 6673.25 \\
    $12.12$ & 6466.89 & 12.28 & 7166.44 \\
    $12.13$ & 6507.55 & 12.29 & 6696.84 \\
    $12.14$ & 6922.77 & 12.30 & 6610.13 \\
    $12.15$ & 6601.03 & 12.31 & 6762.35 \\
    $12.16$ & 6557.66 &  &  \\\hline
    \end{tabular}
    \end{center}
\end{table}

\subsection{Discussion}

The aggregated results derived from the OWA operators and the VGA operator are described in the Figure \ref{Agg_pic}. When set different values of $\alpha$, the different aggregated results will be obtained by OWA operators which means all of these aggregated results are possible. The aggregated results obtained by the VGA operator are neutral that lie in between the aggregated results obtained by the OWA operators when $\alpha {\rm{ = }}0.5$ and when $\alpha {\rm{ = }}0.6$, and graphically, the change tendency of the aggregated results obtained by VGA operator is mainly same as OWA operators' which illustrates VGA operator aggregating data is correct and feasible.

Different from the OWA operators, the VGA operator gets weights without other complicated calculation and the weights obtained by VGA operator are varied. Suppose that there are several time series, if the order of arguments in each time series is different, then the weights will be different. Because the weights are determined according to the degree distribution in the associated graph that means the determination of the weights are relevant with when the argument appears in the histogram. In other words, the VGA operator considers time factor when determining the weights. The VGA operator offers a convenient method to effectively aggregate time series and conserve time information.

\begin{figure}[!t]
\centering
\includegraphics[scale=0.5]{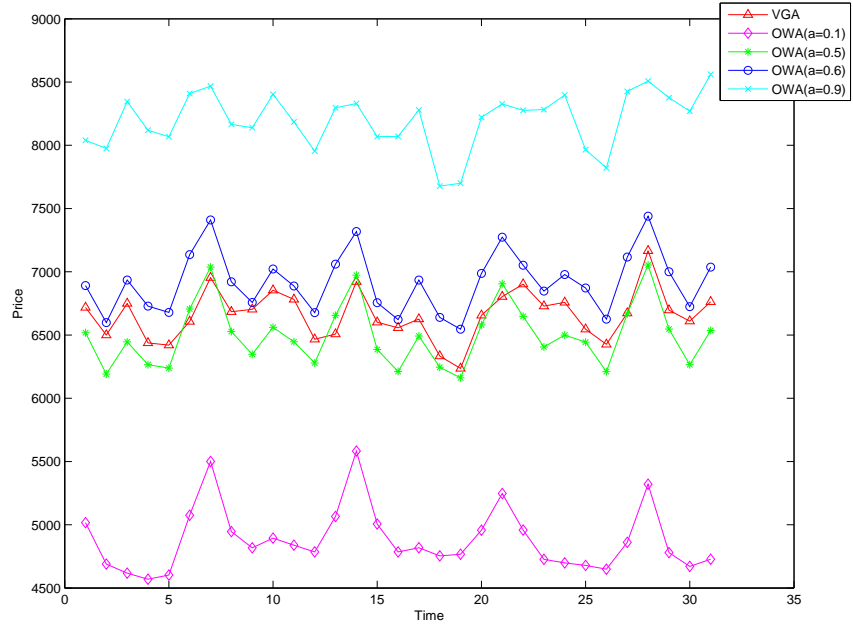}
\caption{Aggregated results}
\label{Agg_pic}
\end{figure}

\section{Conclusion}
\label{conclusion}
In this paper, a visibility graph averaging (VGA) aggregation operator is proposed. This aggregation operator converts data into the visibility graph, then decides the weights according to the degree distribution. The experimental result illustrates that the VGA operator is practical and compared with OWA operators, it shows its advantage that it does not need to calculate weights by other complicated methods and it can aggregate time series effectively, hence, we believe that it can be used in some areas such as economics, space science, weather forecast and so forth where it needs to analyze abundant data of time series.

\section{Acknowledgments}
The work is partially supported by National Natural Science Foundation of China, Grant No. 61174022, Chongqing Natural Science Foundation (for Distinguished Young Scholars), Grant No. CSCT, 2010BA2003, National High Technology Research and Development Program of China (863 Program), Grant No. 2013AA013801.

\bibliographystyle{model1-num-names}
\bibliography{GVA}

\end{document}